\documentclass[letterpaper, 10 pt, conference]{ieeeconf}
\IEEEoverridecommandlockouts
\usepackage[utf8]{inputenc} 
\usepackage{dsfont}
\usepackage{amsfonts}
\usepackage{amsmath}
\usepackage{amssymb}
\usepackage{hyperref}       
\usepackage{url}            
\usepackage{booktabs}       

\usepackage{nicefrac}       
\usepackage{microtype}      

\usepackage{algorithmic}
\usepackage{mathrsfs}
\usepackage{graphicx}
\usepackage{textcomp}
\usepackage{xcolor}

\newtheorem{theorem}{Theorem}
\newtheorem{corollary}{Corollary}
\newtheorem{lemma}{Lemma}

\newtheorem{assumption*}{Assumption}
\newtheorem{stdassumption*}{Standing Assumption}
\newtheorem{stdassumption}{Standing Assumption}

\newtheorem{prop}{Proposition}
\newtheorem{definition}{Definition}
\newtheorem{definition*}{Definition}

\DeclareMathOperator*{\argmin}{arg\,min}
\DeclareMathOperator*{\E}{\mathbb{E}}

\DeclareMathOperator*{\R}{\mathbb{R}}

\usepackage{mathtools}
\graphicspath{ {./images/} }

\usepackage[ruled,vlined]{algorithm2e}
\DeclareMathOperator*{\DP}{\mathrm{D_{\Psi}}\vert}
\DeclareMathOperator*{\sip}{\mathrm{\sigma_{\Psi}}}
\usepackage{cleveref}

\def\BibTeX{{\rm B\kern-.05em{\sc i\kern-.025em b}\kern-.08em
    T\kern-.1667em\lower.7ex\hbox{E}\kern-.125emX}}
\begin{document}

\title{\bf \LARGE 
Is Stochastic Mirror Descent Vulnerable to Adversarial Delay Attacks? 
A Traffic Assignment Resilience Study
}

\author{Yunian Pan, Tao Li, and Quanyan Zhu$^*$
\thanks{$^*$The authors are with the Department of Electrical and Computer Engineering, Tandon School of Engineering, New York University, Brooklyn, NY, 11201 USA; E-mail: {\tt\small \{yp1170,tl2636,qz494\}@nyu.edu}}%
}

\maketitle

\begin{abstract}
 \textit{Intelligent Navigation Systems} (INS) are exposed to an increasing number of informational attack vectors, which often intercept through the communication channels between the INS and the transportation network during the data collecting process. 
To measure the resilience of INS, we use the concept of a Wardrop Non-Equilibrium Solution (WANES), which is characterized by the probabilistic outcome of learning within a bounded number of interactions. 
By using concentration arguments, we have discovered that any bounded feedback delaying attack only degrades the systematic performance up to order $\tilde{\mathcal{O}}(\sqrt{{d^3}{T^{-1}}})$ along the traffic flow trajectory within the Delayed Mirror Descent (DMD) online-learning framework. This degradation in performance can occur with only mild assumptions imposed. 
Our result implies that learning-based INS infrastructures can achieve Wardrop Non-equilibrium even when experiencing a certain period of disruption in the information structure. These findings provide valuable insights for designing defense mechanisms against possible jamming attacks across different layers of the transportation ecosystem.
\end{abstract}


\section{Introduction}


The real-time routing demand has been significantly growing with the rapid development of the modern \textit{Intelligent Navigation Systems} (INS),
in which typical \textit{Online Navigation Platforms} (ONP), such as Google Maps and Waze, receive billions of routing requests per second. 
It is, therefore, crucial to provide reliable and efficient navigation services for active users, such that the ex-post routing regret is small. 
The regret-free routing for individuals gives rise to special macroscopic traffic conditions, commonly known as the \textit{Wardrop equilibrium} (WE) \cite{wardrop1952road} in \textit{congestion games}. The seeking of WE is referred to as \textit{traffic assignment} problem. 
However, the increasing connectivity of transportation networks exposes the INS to a wide variety of {\it informational attacks} \cite{pan2022informational}.




This paper focuses on a class of {\it information-delaying attacks} against the INS that aim to intercept the communication channel between the data source/individual users and the navigation center, \textbf{delaying} the delivery of traffic condition information for adversarial purposes.  
A quintessential attack surface is the data transmission process, including the network jamming attacks \cite{xu2017game} which are often implemented by sending high-frequency wireless interference, or a sheer volume of network packets to the target communication channels/servers.
With the {\it information delays} of critical traffic conditions, the INS infrastructures are at risk of making improper routing recommendations and misguiding users.

Prior studies have shown that relying solely on attack detection and prevention measures is inadequate in the face of pervasive malicious factors \cite{zhu2020cross,ishii2022security}. These findings highlight the need to develop resilient mechanisms to endow traffic systems with self-healing capabilities.
We adopt the notion of \textit{Wardrop Non-Equilibrium Solution} (WANES) \cite{resiliencepaper}, which extends the regret analysis in games \cite{pan2021efficient,tao22confluence} to probabilistic setting, to investigate the provable resilience of \textit{Mirror Descent} (MD) based INS under adversarial delays.
The schematic view of this framework is illustrated in \cref{schematic}.

\begin{figure}
    \centering
    \includegraphics[width=.48\textwidth]{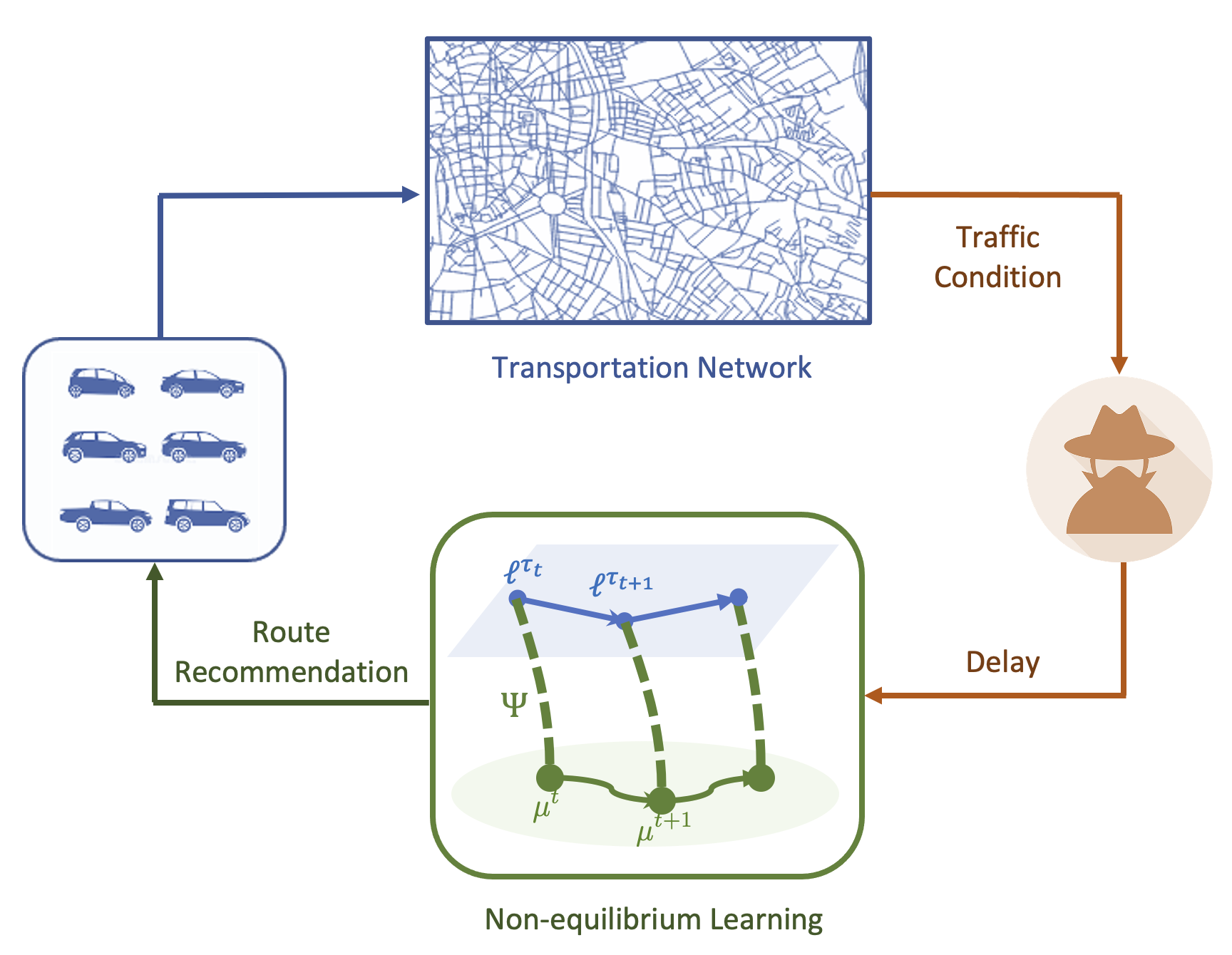}
    \caption{The traffic conditions are withheld at each timestep due to the presence of the attacker, creating a timing disparity between the traffic flow and traffic latency, disrupting the route recommendation. }
    \label{schematic}
\end{figure}

To the best of our knowledge, this  work is among the first endeavors to analyze stochastic mirror descent under adversarial delays, given existing works often focus on the deterministic or bandit setting.
 Our initial step involves analyzing the \textit{Delayed Mirror Descent} (DMD) dynamics. DMD is nearly identical to ordinary MD, but with the added feature of delayed latency feedback. Our analysis rests on the telescoping technique and a Martingale approach, which encounters two significant challenges. The first challenge is establishing the per-iterate telescoping inequality due to the potentially large cardinality (up to $T$) of the delayed latency ``bundle" in the general setting.
 The second challenge is deriving the concentration argument due to the need for special analysis of the maximum of the empirical process. 
We tackle the problems by making additional assumptions about the attack capacity and the uniform bound of the expected latency function.
These assumptions are often practical in traffic assignment problems. 

Without further assumptions imposed, we show that the INS under DMD dynamics with carefully chosen learning parameters endure performance loss up to order $\mathcal{O}(\sqrt{d^3T^{-1}})$ with high probability, which matches the order of delay-free case where $d = 1$. 
When $d = \mathcal{O}(T^{\alpha})$ with $\alpha < 1/3$, the performance loss order indicates a high-probability sub-linear regret bound.
This result can be generalized to resilience analysis of online traffic assignments when the Beckman potential is time-varying.
It further indicates the provability of a certain class of online learning problems with delayed stochastic feedback.

\section{Related Work} \label{rw}

The role of mirror descent in the congestion game frameworks has been discussed in the literature \cite{krichene2014convergence,krichene2015convergence}, where the Hannan consistency in the Ces\'aro sense was established in the deterministic latency setting. 
It was later demonstrated in \cite{Vu2021} that the convergence rate, although remains $\mathcal{O}(T^{-1/2})$ in the stochastic domain, can be lifted to $\mathcal{O}(T^{-2})$ leveraging Nesterov's acceleration scheme in the deterministic domain. 
Our result matches the $T^{-1/2}$ bound in the trivial setting without feedback delays (where $d=1$). 

The gap between online learning and resilience in congestion games was filled in \cite{resiliencepaper}, where the setting was extended by considering the adversary and studying the probabilistic non-equilibrium outcome of learning. 
The adversary is assumed capable of informational manipulation, involving threats both on the physical layers (such as sensors, GPS spoofing \cite{lou2016decentralization},) and cyber layers (such as routing attacks and DoS attacks \cite{mecheva2020cybersecurity}.) 
Our study concerns a general class of attacks that cause delayed feedback to the traffic assignment systems.

There have been versatile delay-handling strategies in the literature. 
One is pooling multiple independent learners together to process the buffered feedback vectors in a sequential manner, see the analysis of Joulani et al. in \cite{joulani2013online}. Another is treating delays as introduced by distributed asynchronous processors, which is explored by Agarwal et al. \cite{NIPS2011_f0e52b27}. 
Both strategies mentioned above can be resource intense in practice. 
Therefore, this work builds upon simplistic models \cite{NIPS2015_advdelay}, and explores the stochastic feedback setting, due to its absence in current literature.

\section{Problem Formulation}\label{pf}

\noindent \textbf{Repeated Congestion Game} A transportation network is abstracted by a directed, finite, and connected graph $\mathcal{G} = (\mathcal{V}, \mathcal{E})$, with nodes $\mathcal{V}$ depicting road junctions, transportation hubs, etc., and edges $\mathcal{E}$ indexing road segments, transportation segments, etc., between different node pairs, we assume that $(v, v) \notin \mathcal{E}$, for all $v \in \mathcal{V}$.
The set of Origin-Destination (OD) pairs is $\mathcal{W} \subseteq \mathcal{V} \times \mathcal{V}$. 
Between each OD pair $w \in \mathcal{W}$  is a set of directed paths $\mathcal{P}_w$, let $\mathcal{P}: = \bigcup_{w\in\mathcal{W}}\mathcal{P}_w$. 

The vehicles over $\mathcal{G}$ constitute a set of infinitesimal players $\mathcal{X}$, split into distinct populations indexed by different OD pairs, i.e., $\mathcal{X} = \bigcup_{w\in\mathcal{W}} \mathcal{X}_w$ and $\mathcal{X}_w \bigcap \mathcal{X}_{w^{\prime}} = \emptyset,\ \forall w, w^{\prime} \in \mathcal{W}$. 
For each $w \in \mathcal{W}$, let $m_w$ be the traffic demand, i.e., the number of vehicles traveling between $w$. 

For all $w \in \mathcal{W}$, each player $x \in \mathcal{X}_w$ is equipped with an action set $\mathcal{P}_w$ and makes decisions repeatedly, at each round $t=1, \ldots, T \in \mathbb{N}_+$, each player is committed to a single path $p \in \mathcal{P}$. 
The action profile can be captured a flow vector $\mu\in \Delta:=\{\mu\in \R^{|\mathcal{P}|}_{\geq 0} | \sum_{p \in \mathcal{P}_w} \mu_p = m_w \ \forall w \in \mathcal{W}\}$. A path flow vector determines an edge flow vector $q \in \R^{|\mathcal{E}|}_{\geq 0 }$, through the edge-path incident matrix $\Lambda = [\Lambda^1 \vert , \ldots, \vert \Lambda^{|\mathcal{W}|}] \in \R^{|\mathcal{E}| \times |\mathcal{P}|}$ such that $\Lambda^w_{e, p} = \mathds{1}_{\{e\in p\}}, \forall e\in \mathcal{E}, w \in \mathcal{W}, p \in \mathcal{P}_w$. In a compact form, $q = \Lambda \mu$.

The total travel time of a road segment is jointly determined by the traffic flow on that road and some stochastic factors, such as weather condition and road incidents, which affects the congestion level. 
Let $(\Omega, \mathcal{F}, \mathbb{P})$ be the underlying probability space; $\omega \in \Omega$ encapsulates the universal latent conditions for $\mathcal{G}$.
Let $l_e: \R_{\geq 0} \times \Omega \mapsto \R_+$ be the edge latency function for $e \in \mathcal{E}$, $l: \R^{|\mathcal{E}|}_{\geq 0} \times \Omega \mapsto \R^{|\mathcal{E}|}_+$ be its vector-valued extension, $\ell:  \Delta \times \Omega \mapsto \R^{|\mathcal{P}|}_+$ be the path latency function. Fixing $\omega \in \Omega$, one can verify that $\ell = \Lambda^{\top} l (\Lambda \mu, \omega)$.

Standing Assumption \ref{standingassumption1}  ensures that the latency functions realistically capture the relation between traffic flow and travel time.

\begin{stdassumption}\label{standingassumption1} 
  The latency functions $l_e$ are $\mathcal{F}$-measurable, differentiable w.r.t. $q_e$, for all $ e \in \mathcal{E}$, and $\cfrac{ \partial l_e (q_e, \omega)}{\partial q_e} > 0 $ for all $q_e \geq 0$.
\end{stdassumption}

Each path flow profile $\mu \in \Delta$ induces a pushforward probability measure $\mathbb{P}_{\ell, \mu} :  \mathcal{B}(\R^{|\mathcal{P}|}_+) \to [0,1]$ associated with the positive random vector $\ell(\mu, \cdot): \Omega \mapsto \R^{|\mathcal{P}|}_+$. 

The Standing Assumption \ref{standingassumption2} quantitatively ensures that the latency for a given path flow is relatively stable in the sense that it has a subgaussian tail. 
In other words, its variance magnitude is controlled by the parameter $\sigma$. 
It also implies that the expected travel time of the paths is bounded by some upper estimate that is linearly related to $\sigma$ by the oracle-given constant $\kappa$. 
 Such a setting is practical in most realistic scenarios. 
\begin{stdassumption}\label{standingassumption2}~
\begin{itemize}
    \item There exists $\sigma > 0$, such that for all $\mu \in \Delta$, the Euclidean norm $\|z\| = \|\ell(\mu) - \E  [\ell (\mu )]\|$ is $\sigma$-subgaussian, i.e., 
    \begin{equation*}
         \E [\exp (\frac{\|z \|^2}{\sigma^2} )] \leq \exp(1).
    \end{equation*}
\item There exists a positive constant $L$, such that, 
\begin{equation*}
    \| \E \ell (\mu)\| \leq L \quad \forall \mu \in \Delta.
\end{equation*} 
 Further, there exists a constant $\kappa > 0$ such that $L \leq \kappa \sigma$.
\end{itemize}

\end{stdassumption}


\noindent \textbf{Resilience under Adversarial Information Delay} Wardrop Equilibrium (WE) has been a conventional solution concept in transportation literature that describes the conditions under which the individual users have the least ex-post regret. 
 Since the latent variable is oftentimes unobservable, we consider a meta-version of the congestion game, $\mathcal{G}_c = (\mathcal{G}, \mathcal{W}, \mathcal{X}, \mathcal{P}, \mathbb{E}[\ell(\cdot)])$, with the utility functions replaced by the expected latency function. This ``meta'' game gives rise to a solution concept corresponding to Definition \ref{wedef}. 

\begin{definition}[Mean Wardrop Equilibrium \cite{wardrop1952road}] 
\label{wedef}
 A path flow $\mu \in \Delta$ is said to be a \textit{Mean Wardrop Equilibrium} (MWE) if $\forall w \in \mathcal{W}$, $ \mu_p > 0$ indicates $ \E [\ell_p] \leq \E[\ell_{p^{\prime}}]$ for all $p, p^{\prime} \in \mathcal{P}_w$. 
 The set of all MWE is denoted by $\boldsymbol{\mu}^*$. 
\end{definition}
It is well known that at MWE, the Beckman Potential of $\mathcal{G}_c$ is minimized; by Fubini's theorem, it is equivalent to minimizing the Mean Beckman Potential (MBP), as in \eqref{beckmanopt},
\begin{equation} \label{beckmanopt}
       \min_{\mu \in \Delta} \Phi (\mu) := \E \left[\sum_{e \in \mathcal{E}}\int_{0}^{(\Lambda\mu)_e} l_e(z, \omega)dz \right]. 
\end{equation}
Given Standing Assumption \ref{standingassumption1}, $\Phi$ is in general non-strictly convex, $\nabla_{\mu} \Phi(\mu)=  \E_{\omega} [ \Lambda^{\top} l (\Lambda \mu, \omega)] = \E [\ell (\mu)]$.

By convention, at each time $t$, within each OD population $w \in \mathcal{W}$, if the infinitesimal players randomize independently according to some mixed strategy $\pi^t_w \in \Delta(\mathcal{P}_w)$ identically, the individual-level and population-level decision makings are equivalent \cite{krichene2014convergence}, in the sense that $ \pi^t_w = \frac{1}{ m_w} (\mu^t_p)_{p \in \mathcal{P}_w }$. 
Let $\mathcal{H}_t$ denote the history of information about $\mu^t$ and $\ell^t$ realizations, a learning algorithm $\mathcal{A}$ maps a $\mathcal{H}_t$ to $\mu^{t+1}$.

Given $\mathcal{A}$, an MWE solution is oftentimes infeasible within finite rounds, yet Wardrop Non-Equilibrium Solutions (WANES) \cite{resiliencepaper} provide additional freedom to analyze the transient behavior.

\begin{definition}
  \label{wanes}
For the congestion game $\mathcal{G}_c$, let $(\Delta^T, \mathcal{B}^T)$ be the product space, with $\mathcal{B}^T$ be the product Borel algebra of $\Delta^T$. For any $\epsilon>0$, define the target set as $\mathcal{C}_\epsilon:=\{\mu\in \Delta| \Phi(\mu)-\Phi^*<\epsilon\}$. A probability measure $\mathbb{P}_T$ over $(\Delta^T, \mathcal{B}^T)$ is an $(\epsilon,\delta)$-Wardrop Non-Equilibrium solution (WANES) if  $\mathbb{P}_T\{(\mu^t)_{t=1}^T\in \Delta^T| \bar{\mu}^T\in \mathcal{C}_{\epsilon}\}\geq 1-\delta$, with $\bar{\mu}^T = \frac{1}{T}\sum_{t=1}^T\mu^t$.
Furthermore, any learning algorithm $\mathcal{A}$ producing such $\mathbb{P}_T$ is said to be $(\epsilon,\delta)$-resilient. 
\end{definition}

\noindent \textbf{Attack Model} 
We consider the online traffic navigation scenario, where a feedback-delaying attacker (typically network jamming attacker) can delay the crowdsourcing of traffic latency (typically by launching DoS attacks), withholding the traffic latency of each round for a finite period of time. 
Consequently, the ONP navigation center is not able to retrieve the exact travel time estimates from the planned paths at each round $t$. 
Instead, at every $t$, a historical ``bundle'' of latency feedback is revealed to the navigation center, while the real-time traffic latency is to be revealed in the future.

Mathematically, let the attacker's action be a vector $\mathbf{d} =  (\min \{ d_t, T - t +1 \})_{1 \leq t \leq T}$, where $d_t \in \mathbb{N}_+$ represents the delayed time that the latency information $\ell^t$ is delivered to the INS. 
The total budget delay is $D = \sum_{t=1}^T\min \{d_t , T - t + 1  \} $, which is of order $\mathcal{O}( T^2 )$, as the maximum $D$ is $T(T+1)/2$.  
We assume that the attacker's maximum per-iterate attack budget is $d:=  \| \mathbf{d} \|_{\infty}  \ll T$. 
The INS receiver's information structure \cite{tao_info} at time $t$ includes the latency vector ``bundle'' indexed by $\mathcal{D}_t = \{ k \vert k + \min\{d_k , T - k + 1 \}-1 = t \}$, i.e., they receive $\mathcal{L}^t := \{ \ell^k |  k \in \mathcal{D}_t\}$, without access to the ``time stamps'' of $\ell^k$.





\section{Traffic Assignment with Delayed Feedback} \label{algsec}

\noindent\textbf{The Delayed Mirror Descent} Let $\bar{\ell}^t$ be the estimate of the sum of arrived latency ``bundle'', i.e., $\bar{\ell}^t =  \sum_{ \tau \in \mathcal{D}_t} \ell^{\tau}.$
The Delayed Mirror Descent (DMD), as shown in Algorithm \ref{delayedfeedbackalgo}, replaces the latency vector in the ordinary Mirror Descent algorithm with $\bar{\ell}^t$.
We use the Bregman divergence that measures the dissimilarity between two iterates.

\begin{definition}
    Given a mirror map $\Psi: \nabla \to \bar{\R}$ and two points $\mu_1, \mu_2 \in \Delta$, the Bregman divergence is $ {\DP}_{\mu_2}^{\mu_1} = \Psi(\mu_1) - \Psi(\mu_2) - \langle \nabla\Psi(\mu_1) , \mu_1 - \mu_2 \rangle$.
\end{definition}

Suppose that the mirror map  $\Psi$ is $\sip$-strongly convex, we have $\DP_{\mu_2}^{\mu_1} \geq \frac{\sip}{2}\|\mu_1 - \mu_2 \|^2$.  We refer readers to \cite{BECK2003167} for more in-depth discussion regarding this notion. 

\begin{algorithm}[htbp]
\label{delayedfeedbackalgo}
\SetKwInOut{Input}{Input}
\SetAlgoLined
\Input{initialize $\mu^1 \in \Delta$, learning rate $\eta$.}
\For{$t \in 1, \ldots, T$}{
     \For{$w \in \mathcal{W}$, $x \in \mathcal{X}_w$, } 
     { 
     INS assigns mixed strategy $\pi^t(\cdot, x) \leftarrow \frac{1}{m_w} (\mu^t_p)_{p \in \mathcal{P}_w}$ to player $x$\;
     player $x$ samples path $A(x) \sim \pi^t(\cdot, x)$;
     }
     Players $\mathcal{X}$ suffer latency $ \ell^t \sim \mathbb{P}_{\ell, \mu^t}(\cdot)$\;
     INS reveals latency vector  $(\ell^k)_{k \in \mathcal{D}_t}$ to $\mathcal{X}$\;
     INS updates:
     \begin{equation}\label{mddyna}
         \mu^{t+1}  \leftarrow \argmin_{\mu \in \Delta}  \langle  \mu , \eta \bar{\ell}^t \rangle + D_{\Psi}(\mu, \mu^t)
     \end{equation}
}
\caption{Delayed Mirror Descent \texttt{(DMD)}}
\end{algorithm}


\noindent\textbf{Telescoping Setup}  Central to the analysis,  Lemma \ref{telescope} quantifies the summation of BMP functional gains along the learning trajectory, setting up conditions for telescoping the sequence. 

Let $\mathcal{F}_t = \sigma( (z_{\tau})_{\tau \in \mathcal{D}_1}, \ldots,  (z_{\tau})_{\tau \in \mathcal{D}_{t-1}})$ be the filtration of the delayed data generating process for $t =1,\ldots, T$, note that $\mu^t$ is $\mathcal{F}_t$-measurable. We fix an equilibrium flow $\mu^* \in \boldsymbol{\mu}^*$ for simplicity of analysis, and define the process $\xi_t = \eta \langle z^t, \mu^* - \mu^t \rangle$ for $t =1, \ldots, T$. We also let $\|z^{t}_{m}\|:= \max_{r \in \cup_{s =\tau_t}^{t}  \mathcal{D}_s} \|z^{r}\|$ be the maximum of the empirical process generated through $\mathcal{D}_{\tau_t}, \ldots, \mathcal{D}_t$, with $\tau_t = \min \mathcal{D}_t$. 

\begin{lemma} \label{telescope}
Under the mirror descent dynamics with latency feedback, the following holds for $t = 1, \ldots, T$,
\begin{equation} \label{telescopingsetup}
\begin{aligned}
     & \sum_{\tau \in \mathcal{D}_t} \eta\left(\Phi^{\tau}-\Phi^*\right) - \frac{2\eta^2 d}{\sigma_{\Psi}}L^2  + {\DP}^{\mu^{*}}_{\mu^{t+1}}-{\DP}^{\mu^{*}}_{\mu^t}\\ &  \leq  \sum_{\tau \in \mathcal{D}_t}\xi_{\tau} + \frac{2\eta^2d}{\sigma_{\Psi}} \| z^t_m \|^2  .
\end{aligned}
\end{equation}
\end{lemma}
 As illustrated in Fig. \ref{cardinbounded}, due to bounded attack capacity $d$, the cardinality of the delayed latency ``bundle'' $|\mathcal{D}_t| \leq d$ for every $t \leq 1, \ldots, T$. One can further see that the cardinality of $ \cup_{s =\tau_t}^t \mathcal{D}_t $ should be bounded by $2d$. 

\begin{figure}[htbp]
    \centering
    \includegraphics[width = .35\textwidth]{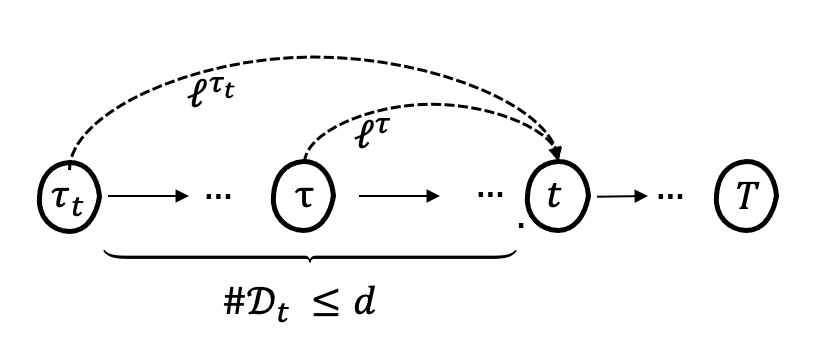}
    \caption{The per-iterate latency ``bundle" is bounded, $t - \tau_t + 1\leq d$.}
    \label{cardinbounded}
\end{figure}

\begin{proof} \label{telescopingproof}
Recall that the mirror step can be equivalently written as 
\begin{equation*}
     \mu^{t+1} = :\nabla  \Psi^* (  \underbrace{\nabla \Psi(\mu^t) - \sum_{\tau \in \mathcal{D}_t}\eta \ell^{\tau}}_{\text{dual step }} ),
\end{equation*}
where $\Psi^*: \R^{|\mathcal{P}|} \to \R $ is the Fenchel conjugate of the mirror map $\Psi$. 
Under proper conditions, for $\mu \in \partial \Psi^*(\nu) $, $\nu \in \partial \Psi(\mu)$, it holds that $  \Psi(\mu) + \Psi^*( \nu) = \langle \mu,   \nu \rangle $. Therefore, given a primal flow point $\mu^t$, we define $ \nu^{t+1} := \nabla \Psi(\mu^t) - \sum_{\tau \in \mathcal{D}_t}\eta \ell^{\tau} = \nabla \Psi(\mu^{t+1})$ as the dual latency point. One can verify that $\mu^{t+1} = \nabla \Psi^*(\nu^{t+1})$ and the dual step can be written as $\nu^{t+1} = \nu^t - \sum_{\tau \in \mathcal{D}_t} \eta \ell^{\tau}$. 

Due to the convexity of $\Phi$,
\begin{equation*}
\begin{aligned}
   &  \quad \sum_{\tau \in \mathcal{D}_t}  \eta ( \Phi(\mu^{\tau}) - \Phi^*)  \leq   \sum_{\tau \in \mathcal{D}_t } \eta \langle \E \ell^{\tau}, \mu^{\tau} - \mu^* \rangle \\
   & =   \sum_{\tau \in \mathcal{D}_t } \eta \langle z^{\tau}, \mu^* - \mu^{\tau} \rangle + \eta \langle \ell^{\tau}, \mu^{\tau} - \mu^* \rangle \\
     &  =  \sum_{\tau \in \mathcal{D}_t }  \xi_{\tau} +  \sum_{\tau \in \mathcal{D}_t}\eta\langle\ell^{\tau}  , \mu^{\tau,-} - \mu^* \rangle   + \eta\langle  \ell^{\tau} , \mu^{\tau}- \mu^{\tau,-} \rangle , 
\end{aligned}
\end{equation*}
where we let $\mu^{\tau,-}$ be intermediate primal flow before applying latency $ \ell^{\tau}$, i.e., let $\mathcal{D}_{t, \tau} := \{ r \in \mathcal{D}_{t}, r < \tau\}$, $\mu^{\tau,-} := \nabla \Psi^*( \nu^{\tau,-} ) = \nabla \Psi^*( \nabla \Psi(\mu^{t}) - \sum_{ r \in \mathcal{D}_{t, \tau}}\eta \ell^r)$. Similarly we define the immediate primal flow point after applying latency $\ell^{\tau}$ as $\mu^{\tau,+} :=  \nabla \Psi^*( \nu^{\tau,+}) :=  \nabla \Psi^*( \nabla \Psi(\mu^{\tau,-}) - \eta \ell^{\tau})$.
 One can then verify that $\mu^{\tau_t, -} = \mu^t $ and $\mu^{\max \mathcal{D}_t, +} = \mu^{t+1}$.

Consider the following decomposition for arbitrary $\mu^* \in \Delta$, $ \langle\mu^{\tau,-} - \mu^*, \ell^{\tau} \rangle = \langle\mu^{\tau,-} - \mu^{\tau,+}, \ell^{\tau} \rangle + \langle\mu^{\tau,+} - \mu^*, \ell^{\tau} \rangle$. 
By the first-order optimality condition of \eqref{mddyna}, $$\langle   \eta \ell^{\tau}  + \nabla \Psi (\mu^{\tau,+})-\nabla \Psi(\mu^{\tau,-}) ,  \mu^{\tau,+} - \mu^* \rangle \leq 0.$$ Apply Pythagorean identity of Bregman divergence: ${\DP}^{\mu^{\tau,+}}_{\mu^{\tau,-}}+{\DP}^{\mu^{*}}_{ \mu^{\tau,+}} - {\DP}^{\mu^{*}}_{\mu^{\tau,-}}=\langle\nabla \Psi(\mu^{\tau,+})-\nabla \Psi(\mu^{\tau,-}), \mu^{\tau,+}-\mu^{*} \rangle$, and we arrive at 
\begin{equation*}
    \eta\langle  \ell^{\tau}  ,  \mu^{\tau,+} - \mu^* \rangle  + {\DP}^{\mu^{\tau,+}}_{\mu^{\tau,-}}+{\DP}^{\mu^{*}}_{\mu^{\tau,+}}-{\DP}^{\mu^{*}}_{ \mu^{\tau,-}} \leq 0.
\end{equation*}
Summing over $\tau \in \mathcal{D}_t$, since ${\DP}^{\mu^{\tau,+}}_{\mu^{\tau,-}} \geq \frac{\sigma_{\Psi}}{2}\| \mu^{\tau,+} - \mu^{\tau,-}\|^2$, by Cauchy-Schwarz inequality and $-b^2 + 2ab \leq a^2$, 
\begin{equation*}
    \begin{aligned}
       & \quad  \sum_{\tau \in \mathcal{D}_t}\eta\langle  \ell^{\tau}   ,  \mu^{\tau,-} - \mu^* \rangle + \frac{\sigma_{\Psi}}{2}\| \mu^{\tau,+} - \mu^{\tau,-}\|^2  \\ & \leq  \sum_{\tau \in \mathcal{D}_t}{\DP}^{\mu^{*}}_{\mu^{\tau,-}}-{\DP}^{\mu^{*}}_{\mu^{\tau,+}}  + \langle\mu^{\tau,-} - \mu^{\tau,+}, \eta \ell^{\tau} \rangle 
       \\ & \leq  \sum_{\tau \in \mathcal{D}_t}{\DP}^{\mu^{*}}_{\mu^{\tau,-}}- {\DP}^{\mu^{*}}_{\mu^{\tau,+}}.  + \eta\| \ell^{\tau} \| \|\mu^{\tau,-} - \mu^{\tau,+} \|, 
    \end{aligned}
\end{equation*}
which then gives
       \begin{equation*}
           \begin{aligned} & \quad  \sum_{\tau \in \mathcal{D}_t}\eta\langle  \ell^{\tau}   ,  \mu^{\tau,-} - \mu^* \rangle \\
         & \leq \sum_{\tau \in \mathcal{D}_t}{\DP}^{\mu^{*}}_{\mu^{\tau,-}}-{\DP}^{\mu^{*}}_{\mu^{\tau,+}} + \frac{\eta^2}{ 2 \sigma_{\Psi}} \|  \ell^{\tau}\|^2 
         \\ & \leq \sum_{\tau \in \mathcal{D}_t}{\DP}^{\mu^{*}}_{\mu^{\tau,-}}-{\DP}^{\mu^{*}}_{\mu^{\tau,+}} + \frac{\eta^2}{ 2 \sigma_{\Psi}} (L^2 + \| z^{\tau}\|^2).
    \end{aligned}
\end{equation*}

Now we analyze $\eta\langle  \ell^{\tau} , \mu^{\tau}- \mu^{\tau,-} \rangle$: applying Cauchy-Schwarz inequality and  triangular inequality several times,
\begin{equation*}
    \begin{aligned}
         & \quad \sum_{\tau \in \mathcal{D}_t } \eta\langle  \ell^{\tau} , \mu^{\tau}- \mu^{\tau,-} \rangle  \leq \sum_{\tau \in \mathcal{D}_t } \eta \|\ell^{\tau} \| \|  \mu^{\tau} - \mu^{\tau,-}\| \\ & \leq \sum_{\tau \in \mathcal{D}_t } \eta \|\ell^{\tau} \|   \left(\sum_{s = \tau}^{t-1} \| \mu^s - \mu^{s+1}  \| + \| \mu^t - \mu^{\tau,-} \| \right) 
         \\ 
         & =  \sum_{\tau \in \mathcal{D}_t } \eta  \| \ell^{\tau}\| 
          \bigg(  \sum_{s=r}^{t-1} \left\| \nabla \Psi^*(\nu^s)  - \nabla \Psi^*(\nu^{s+1}) \right\| \\ 
          & \quad \quad  \quad \quad \quad \quad \quad + \left\|  \nabla \Psi^*(\nu^t)  - \nabla \Psi^*(\nu^{\tau,-}) \right\| \bigg) .
    \end{aligned}
\end{equation*}
Since $\Psi$ is $\sip$-strongly convex, its Fenchel conjugate $\Psi^*$ is $\frac{1}{\sip}$-smooth, the R.H.S. becomes
\begin{align*} 
       & \leq  \sum_{\tau \in \mathcal{D}_t } \eta  \| \ell^{\tau}\|
          \left(  \frac{1}{\sigma_{\Psi}} ( \sum_{s = \tau}^{t-1} \left\| \nu^s - \nu^{s+1} \right\| 
 + \left\| \nu^t - \nu^{\tau,-}\right\|) \right) 
 \\  & \leq   \sum_{\tau \in \mathcal{D}_t }  \frac{\eta^2}{\sigma_{\Psi}} \| \ell^{\tau}\| \left(  \sum_{s = \tau}^{t-1}\sum_{r \in \mathcal{D}_s}\left\| \ell^r \right\| +  \sum_{ p \in \mathcal{D}_{t, \tau}}  \left\| \ell^p \right\| \right) .
\end{align*}

 To associate the upper estimate with the cardinality of $|\mathcal{D}_t|$,  by triangular inequality, $\| \ell^t \| = \| \E \ell^t + z^t\| \leq L + \|z^t\|$ for all $t$. Breaking the brackets we get, the R.H.S. becomes
 \begin{equation*}
 \begin{aligned}
        &  \leq \underbrace{\frac{\eta^2 L^2 }{ \sigma_{\Psi}} \sum_{\tau \in \mathcal{D}_t }  Q_{\tau}}_{\text{I}}    + \underbrace{\frac{\eta^2 L}{ \sigma_{\Psi}} \sum_{\tau \in \mathcal{D}_t}\| z^{\tau} \| Q_{\tau}}_{\text{II}} \\
         & \quad + \underbrace{\frac{\eta^2 }{ \sigma_{\Psi}} \sum_{\tau \in \mathcal{D}_t}\| z^{\tau}\|  \left(\sum_{s = \tau}^{t-1}\sum_{r \in \mathcal{D}_s} \| z^r\| + \sum_{p \in \mathcal{D}_{t, \tau}} \|z^p\| \right)}_{\text{III}},
 \end{aligned}
 \end{equation*}
where $ Q_{\tau} =  | \mathcal{D}_{t, \tau}| + \sum_{s = \tau}^{t-1} |\mathcal{D}_s|$, which essentially counts for the number of all latency vectors other than $\ell^{\tau}$, that have been delivered between round $\tau_t$ and $t$.

We fix a $\tau $ and a $t$ and look into round $s \in \{\tau, \ldots, t\}$, when $s = t$, consider $q \in \mathcal{D}_{t, \tau}$; when $s< t$, consider $q \in \mathcal{D}_{s}$.
There are two cases, $q < \tau$ and $q \geq \tau$. Consider both cases quantitatively, $Q_{\tau} \leq Q_{\tau,1} + Q_{\tau,2}$:
 \begin{align*}
   Q_{\tau} & \leq
    \left(\sum_{q \in \mathcal{D}_{t, \tau}} \mathds{1}_{\{q \geq \tau \}}  + \sum_{s = \tau}^{t-1}\sum_{q \in \mathcal{D}_{s} } \mathds{1}_{\{q \geq \tau \}} \right)
     \\
    & \quad +\left( \sum_{q \in \mathcal{D}_{t, \tau}}\mathds{1}_{\{q < \tau \} }  + \sum_{s = \tau}^{t-1} \sum_{q \in \mathcal{D}_{s} }  \mathds{1}_{\{ q < \tau \}} \right).
 \end{align*}

When $q \geq \tau$, by a pigeonhole argument, there are at most $d_{\tau}$ instances, as fixing a $q$, $ q + d_q - 1$ is at most in only one of $\{ q, \ldots, t-1\}$.
Analytically, we have $Q_{\tau,1} = \sum_{q \in \mathcal{D}_{t, \tau}} \mathds{1}_{\{q \geq \tau\}} + \sum_{s = \tau}^{t-1}\sum_{q \in \mathcal{D}_{s} } \mathds{1}_{\{q \geq \tau \}} 
$, which is just $\sum_{s= \tau}^{t-1} \sum_{q =\tau}^{ s } \mathds{1}_{\{  q + d_q -1 = s\}}$. We rearrange the sum and rewrite it as $\sum_{q = \tau}^{t-1}  \sum_{s = q}^{ t-1} \mathds{1}_{\{ q + d_q - 1 = s\}}$ which by observation is bounded by $  d_{\tau} - 1 \leq d_{\tau} $. 

 When $q < \tau$, fixing $q$, we have that $Q_{\tau,2}$ can be written as $
  \quad  \sum_{q \in \mathcal{D}_{t, \tau}}\mathds{1}_{\{q < \tau \} }  + \sum_{s = \tau}^{t-1} \sum_{q \in \mathcal{D}_{s} }  \mathds{1}_{\{ q < \tau \}} $  which is essentially $\sum_{s = \tau}^t |\mathcal{D}_{s,\tau}| $ and can be further written as $\sum_{ q = 1}^{ \tau -1 } \mathds{1}_{\{ q + d_q -1 = t \}}  + \sum_{s= \tau}^{t-1} \sum_{q =1}^{\tau - 1} \mathds{1}_{ \{ q + d_q - 1 = s \} } $. This quantity essentially counts the $q$'s that get delayed into the range $\{\tau, t\}$, i.e., 
$  \sum_{q = 1}^{\tau -1} \mathds{1}_{\{ q + d_q - 1 \in \{ \tau, \ldots, t\}\}}   \leq  d.$



Hence, we arrive at:
\begin{align*}
    \text{I} & \leq \frac{\eta^2_1 L^2 }{ \sigma_{\Psi}}\sum_{\tau \in \mathcal{D}_t }  d_{\tau} + d  \leq \frac{2\eta^2_1 L^2 }{ \sigma_{\Psi}}\sum_{\tau \in \mathcal{D}_t } d,
     \\ \text{II} &\leq  \frac{2\eta^2 L}{ \sigma_{\Psi}} \max_{\tau \in \mathcal{D}_t} \|z^{\tau}\| \sum_{\tau \in \mathcal{D}_t}d,
    \\  \text{III} &\leq \frac{2\eta^2 }{ \sigma_{\Psi}} \max_{\tau \in \cup_{s=\tau_{\min}}^{t}\mathcal{D}_s} \|z^{\tau}\|^2 \sum_{\tau \in \mathcal{D}_t} d.
\end{align*}

Let $\|z^{t}_{m}\|:= \max_{\tau \in \cup_{s =\tau_t}^t \mathcal{D}_s} \|z^{\tau}\|$ be the maximum of the empirical process generated by $\mathcal{D}_{\tau_t}, \ldots, \mathcal{D}_t$, clearly, $\max_{s\in \mathcal{D}_t} \|z^s\| \leq \|z^t_m\|$.
Combining I, II, and III, we obtain
\begin{align*}
    & \quad \sum_{\tau \in \mathcal{D}_t} \eta\left(\Phi^{\tau}-\Phi^*\right)  + {\DP}^{\mu^{*}}_{\mu^{t+1}}-{\DP}^{\mu^{*}}_{\mu^t} \leq  \sum_{\tau \in \mathcal{D}_t } \xi_{\tau} \\ &  + \frac{\eta^2}{2 \sigma_{\Psi}} (L^2 +  \| z^{\tau}\|^2 
     + 2dL^2 + \max_{\tau \in \mathcal{D}_t} 2dL\| z^{\tau}\| + 2d \|z^t_m \|^2)
    \\ & \leq \sum_{\tau \in \mathcal{D}_t } \xi_{\tau} + \frac{\eta^2}{2 \sigma_{\Psi}} ((1 + 2d)(L^2 +\|z^t_m\|^2)  + 2dL \|z^t_m\| )
      \\ & \leq \sum_{\tau \in \mathcal{D}_t } \xi_{\tau} + \frac{2\eta^2d }{\sigma_{\Psi}}( L^2 +\| z_m^t\|^2 ). 
 \end{align*}
Rearrange the terms and we arrive at \eqref{telescopingsetup}.

\end{proof}

\section{Resilience Analysis} \label{ra}
\noindent \textbf{Bounding the Moment Generating Function} To verify that iterates \eqref{mddyna} lead to a Wardrop Non-equilibrium solution, we need to derive a high probability bound for the functional gaps $\Phi(\mu^t) - \Phi^*$ of the flow trajectory $(\mu^t)_{t=1}^{T+1}$ along the learning process. 
To this end, we define a set of weights $\{w_t\}_{t=1}^{T+1}$ that serves as the set of variable coefficients inside the moment generating functions of the functional gaps, which allows for the flexibility of compensating the learning rate $\eta$. 
The target of our analysis is the two auxiliary quantities as we have defined in \eqref{auxqtt}.
\begin{equation}\label{auxqtt}
    \begin{aligned}
   Z_t = & w_{t+1}  \sum_{\tau \in \mathcal{D}_t} (\eta (\Phi^{\tau}-\Phi^*) - \frac{2\eta^2 d L^2}{\sigma_{\Psi}} ) 
   \\ & +  w_{T+1} ({\DP}^{\mu^*}_{\mu^{t+1}} - {\DP}^{\mu^*}_{\mu^t}) \quad \text{for } t = 1, \ldots, T.
   \\ S_t = &\sum_{i=t}^T Z_i , \quad \quad \quad\quad \quad\quad\quad\quad  \text{for } t = 1, \ldots, T+1.
\end{aligned}
\end{equation}

We analyze the moment generating function of $S_t$ conditioned on $\mathcal{F}_{\tau_t }$. By convention, we let $ \tau_t = t$ if $\mathcal{D}_t = \emptyset$, thus $\tau_1 = 1$. 
The core result, as shown in Theorem \ref{inductthm} is a Chernoff-type of bound that gives rise to the main concentration argument of our interest.
\begin{theorem} \label{inductthm} 
Suppose that $\{w_t\}$ satisfies that, $w_{t+1} + w^2_{t+1} \frac{648  d^3 \eta^2 \sigma^2}{\sigma_{\Psi}}  \leq w_{t}$, and $ w_{t+1}\eta^2 d^2 \leq \frac{\sigma_{\Psi}}{432 d  \sigma^2}$. Then, it holds that for every $1 \leq t \leq T$ with probability $1$, 
 \begin{equation}
     \begin{aligned}
  \E [ \exp(S_t) | \mathcal{F}_{\tau_t}] 
    \leq \exp( (w_t - w_{T+1}) {\DP}^{\mu^*}_{\mu^t} + C \sum_{i=t}^T w_{i+1} \eta^2 ), 
 \end{aligned} 
 \end{equation}
where the constant $C := 324 \sigma^2 d^3\sigma_{\Psi}^{-1}(  8 + \kappa^2 ) \big)$.
\end{theorem}

The proof, which is deferred to Appendix, relies on an induction approach. 
The intuition behind this approach is that propagation of subgaussian behavior scales with $d^2$, while the subgaussianity of the maxima of empirical process scales with $d$. Thus, the stochastic error of the per-iterate upper estimate scales with $d^3$.

\noindent\textbf{DMD attains Non-equilibrium} With all the preparations above, we now turn to the non-equilibrium analysis. Corollary \ref{mainresult} comes from the fact that the stochastic fluctuation will be absorbed by the sequence $\{w_t\}$ under certain conditions.

\begin{corollary}\label{mainresult}
Let $w_{T+1} = \frac{\sigma_{\Psi}}{1296d^3\sigma^2\eta^2 (T+1) }$ and $ w_t = w_{t+1} +  \frac{648 w^2_{t+1} d^3 \eta^2}{\sigma_{\Psi}} $ for all $1 \leq t \leq T$. The sequence $\{ w_t \}$ satisfies the condition required by Theorem \ref{inductthm}, and for $\delta \in (0, 1)$, the following events hold with probability at least $1 - \delta$: 
\begin{equation}\label{resultwithlr}
    \frac{1}{T} \sum_{t=1}^T (\Phi(\mu^t) - \Phi^*) \leq \mathcal{O} \left( \frac{{\DP}^{\mu^*}_{\mu^1}}{\eta T} +  \frac{\sigma^2}{\sigma_{\Psi}} d^3 (1 + \ln(\frac{1}{\delta}))\eta \right),
\end{equation}
and 
\begin{equation}\label{resultbregman}
    {\DP}^{\mu^*}_{\mu^{T+1}} \leq \mathcal{O}\left({\DP}^{\mu^*}_{\mu^1} + \frac{\sigma^2}{\sigma_{\Psi}} d^3 ( 1 + \ln(\frac{1}{\delta}))\eta^2\right).
\end{equation}
 Setting $\eta = \sqrt{ \frac{ {\DP}^{\mu^*}_{\mu^{1}}}{ \frac{\sigma^2}{\sigma_{\Psi}} d^3 (1 + \ln(\frac{1}{\delta}))T}}$, we have
 \begin{equation}\label{resultsequence}
    \frac{1}{T} \sum_{t=1}^T (\Phi(\mu^t) - \Phi^*) \leq \tilde{\mathcal{O}}\left( d^{\frac{3}{2}} \sqrt{\frac{ \frac{\sigma^2}{\sigma_{\Psi}} {\DP}^{\mu^*}_{\mu^1}   (1 + \ln(\frac{1}{\delta}))}{T}}\right).
 \end{equation}
\end{corollary}

\begin{proof}
For simplicity, let $c_d := 108d^3$. We first show that the sequence $\{w_t\}$ satisfies the conditions required by Theorem \ref{inductthm}: 
 \begin{align*}
w_{t+1} + w^2_{t+1} \frac{6  c_d \eta^2 \sigma^2}{\sigma_{\Psi}}  \leq w_{t},   \frac{w_{t+1}\eta^2}{\sigma_{\Psi}} \leq \frac{1}{  4c_d \sigma^2}.
 \end{align*}
 Let $A = 6 c_d \sigma^2 \sigma_{\Psi}^{-1} \eta^2 (T+1 )$. Set $w_{T+1} = \frac{1}{2A}$. For $1 \leq t \leq T$, set $w_t$ such that $w_{t+1} + w^2_{t+1} \frac{6  c_d \eta^2 \sigma^2}{\sigma_{\Psi}} = w_{t}$, the first condition is automatically satisfied. 
 To verify the second condition, notice that in this setup, $w_t\eta^2 \leq \frac{\eta^2}{A} = \frac{\sigma_{\Psi}}{  6 c_d \sigma^2 (T+1 )} \leq \frac{\sigma_{\Psi}}{  4c_d\sigma^2 }$. Now let $K =  (w_1 - w_{T+1}){\DP}^{\mu^*}_{\mu^1} + C \sum_{t=1}^Tw_{t+1} \eta^2  + \ln (\frac{1}{\delta})$. By Markov inequality and Theorem \ref{inductthm},
     \begin{align*}
        & \quad \mathbb{P} [S_1 \geq K ] 
        \\ & \leq \exp(-K) \E[ \exp (S_1)] 
         \\ & \leq \exp(-K) \exp((w_1 - w_{T+1}){\DP}^{\mu^*}_{\mu^1} + C \sum_{t=2}^{T+1}w_{t} \eta^2 )
         \\ &=  \delta.
     \end{align*}
      Since 
      $   S_1 = \sum_{t=1}^T w_{t+1} \eta \sum_{\tau \in \mathcal{D}_t} (\Phi(\mu^{\tau}) - \Phi^*)  - \frac{ 2d L^2}{\sigma_{\Psi}}\sum_{t=1}^T w_{t+1} \eta^2
        + w_{T+1} ({\DP}^{\mu^*}_{\mu^T} - {\DP}^{\mu^*}_{\mu^1}),$
      we have that with probability at least $1 - \delta$, 
      \begin{equation*}
        \begin{aligned}
           & \quad  \sum_{t=1}^T w_{t+1} \eta (\sum_{\tau \in \mathcal{D}_t} \Phi(\mu^{\tau}) - \Phi^*)  + w_{T+1}{\DP}^{\mu^*}_{\mu^{T+1}}
            \\ & \leq w_1{\DP}^{\mu^*}_{\mu^{1}} + (\frac{ 2d L^2}{\sigma_{\Psi}} + C)\sum_{t=1}^T w_{t+1} \eta^2  + \ln(\frac{1}{\delta}) .
        \end{aligned}
      \end{equation*}
    Since $w_{T+1} = \frac{1}{2A}$ and $\frac{1}{2A} \leq w_t \leq \frac{1}{A}$ for $1 \leq t \leq T+1$, we plug them into above and obtain
      \begin{align*}
         &\quad \ \eta \sum_{t=1}^T \sum_{\tau \in \mathcal{D}_t}(\Phi(\mu^{\tau}) - \Phi^*) + {\DP}^{\mu^*}_{\mu^{T+1}}  
          \\ & \leq 2 {\DP}^{\mu^*}_{\mu^{1}} + 2(\frac{ 2d L^2}{\sigma_{\Psi}} + C) \eta^2 T + 2A \ln(\frac{1}{\delta})
         \\ & \leq  2 {\DP}^{\mu^*}_{\mu^{1}} +  2(\frac{\sigma^2}{\sigma_{\Psi}} ( B + A \ln(\frac{1}{\delta})))\eta^2 T ,
      \end{align*}
      where $B = \mathcal{O}(d^3)$. Dividing both side by $\eta$ yields
      \begin{equation*}
      \begin{aligned}
            \frac{1}{T} \sum_{t=1}^T (\Phi(\mu^t) - \Phi^*) 
            &  \leq \frac{{\DP}^{\mu^*}_{\mu^1}}{\eta T} +  2(\frac{\sigma^2}{\sigma_{\Psi}} B + A\ln (\frac{1}{\delta}) )\eta
       \\ & \leq \mathcal{O} \left( \frac{{\DP}^{\mu^*}_{\mu^1}}{\eta T} +  \frac{\sigma^2}{\sigma_{\Psi}} d^3 (1 + \ln(\frac{1}{\delta}))\eta \right).
      \end{aligned}
      \end{equation*}
      and ${\DP}^{\mu^*}_{\mu^{T+1}} \leq 2 {\DP}^{\mu^*}_{\mu^{1}} +  2(\frac{\sigma^2}{\sigma_{\Psi}} ( B + A \ln(\frac{1}{\delta})))\eta^2 T$.
      Setting $\eta = \sqrt{ \frac{ {\DP}^{\mu^*}_{\mu^{1}}}{ \frac{\sigma^2}{\sigma_{\Psi}} d^3 (1 + \ln(\frac{1}{\delta}))T}}$ gives the results.

\end{proof}
Corollary
 \ref{mainresult} immediately implies the  resilience in the non-equilibrium sense. 
\begin{prop}
For $\delta \in (0,1)$, the DMD Algorithm \ref{delayedfeedbackalgo} with $\eta = \mathcal{O}(\sqrt{ d^{-3}T^{-1}})$ is $( \epsilon, \delta)$-resilient, which gives a $(\epsilon, \delta)$-WANES, with $\epsilon = \tilde{\mathcal{O}}(\sqrt{\frac{d^3}{T}})$.
\end{prop}

 \begin{proof}
     For $\mu^1, \ldots, \mu^T$ produced by the DMD algorithm, by the convexity of $\Phi^*$, $\Phi(\bar{\mu}^T) - \Phi^* \leq \frac{1}{T}\sum_{t=1}^T \Phi(\mu^t) - \Phi^*$ which satisfies Corollary \ref{mainresult} with $1 - \delta$, hence the statement follows.

 \end{proof}



\section{Conclusion} \label{cc}

In this paper, we have investigated the resilience of DMD-based INS under adversarial delay attacks. 
We made some mild assumptions to handle the challenges that arose in finite-time analysis, obtaining a high probability bound for the performance loss.
With the aid of the non-equilibrium notion, we have demonstrated the self-restoring capability of INS to recover from information-delaying attacks.

Future research would focus on developing scalable and distributed strategies to handle adversarial delays to improve the defense mechanism in the face of cyber-physical threats.
We would also refine the analysis of concentration arguments, improving the order of existing results to match the lower bound in the deterministic setting. 

\bibliographystyle{unsrt}
\bibliography{reference}

\appendix

\begin{proof}[Proof of Theorem \ref{inductthm}]
\label{chernoffproof}
      We proceed by induction on $t$. Consider the base case $t = T + 1$, $S_t = 0$ and $(w_t - w_{T+1}){\DP}^{\mu^*}_{\mu^t} = 0$, so the inequality follows. 
      Consider $1 \leq t \leq T$, suppose $\tau_{t+1} > \tau_t$ (this assumption is dispensable but simplifies the analysis); we have 
      \begin{equation} \label{induct0}
      \begin{aligned}
          \E [\exp(S_t) | \mathcal{F}_{\tau_t}] &= \E [\exp(Z_t + S_{t+1})|\mathcal{F}_{\tau_t}] \\ &= \E \E [\exp(Z_t + S_{t+1})|\mathcal{F}_{\tau_{t+1}}] | \mathcal{F}_{\tau_t} ].
      \end{aligned}
    \end{equation}
    If $ \tau_{t+1} < \tau_t$, we should have but the argument holds similarly as in this case $\mathcal{F}_{\tau_{t+1}} \subseteq \mathcal{F}_{\tau_t}$. 
    Now suppose the inequality holds for $t+1$, $Z_t$ is determined given $\mathcal{F}_{\tau_{t+1}}$.
    By the induction hypothesis, we have, w.p. 1, 
    \begin{equation} \label{induct1}
    \begin{aligned}
       & \quad \E [\exp(Z_t + S_{t+1})|\mathcal{F}_{\tau_{t+1}}]  \leq \exp(Z_t) \\& \exp\left( (w_{t+1} - w_{T+1}) {\DP}^{\mu^*}_{\mu^{t+1}} +  C \sum_{i=t+1}^T w_{i+1} \eta^2 \right).
    \end{aligned}
    \end{equation}
    By Lemma \ref{telescope} we have,
    \begin{align*}
        & \quad \sum_{ \tau \in\mathcal{D}_t}(\eta (\Phi^{\tau}-\Phi^*) - \frac{2 \eta^2 d L^2}{\sigma_{\Psi}} ) \\ & \leq 
         \sum_{\tau \in \mathcal{D}_t} \xi_{\tau} - ({\DP}^{\mu^{*}}_{\mu^{t+1}} -{\DP}^{\mu^{*}}_{\mu^t}) + \sum_{\tau \in \mathcal{D}_t } \frac{2 \eta^2 d}{ \sigma_{\Psi}}  \| z^t_m\|^2 ,
    \end{align*}
    and thus,
   \begin{align*}
    Z_t & \leq w_{t+1} (\sum_{\tau \in \mathcal{D}_t} \xi_{\tau} -({\DP}^{\mu^{*}}_{\mu^{t+1}} - {\DP}^{\mu^{*}}_{\mu^t})
    \\ &   \quad +  \frac{2\eta^2d}{ \sigma_{\Psi}}  \| z^t_m\|^2)   + w_{T+1} ({\DP}^{\mu^{*}}_{\mu^{t+1}} - {\DP}^{\mu^{*}}_{\mu^t}) 
    \\ & = w_{t+1} (\sum_{\tau \in \mathcal{D}_t} \xi_{\tau}+ \frac{2\eta^2d}{ \sigma_{\Psi}}  \| z^t_m\|^2)
    \\ & \quad - (w_{t+1} - w_{T+1} ) ({\DP}^{\mu^{*}}_{\mu^{t+1}} -{\DP}^{\mu^{*}}_{\mu^t}) .
   \end{align*}    
    Plugging into \eqref{induct1}, we obtain
    \begin{align*}
       & \quad  \E [\exp(Z_t + S_{t+1}) | \mathcal{F}_{\tau_{t+1}} ] \\& \leq \exp \bigg( (w_{t+1} - w_{T+1}) {\DP}^{\mu^*}_{\mu^t} + w_{t+1} (\sum_{\tau \in \mathcal{D}_t} \xi_{\tau }   \\ & \quad +  \frac{2\eta^2d}{ \sigma_{\Psi}}  \| z^t_m\|^2) + C \sum_{i=t+1}^T w_{i+1}\eta^2 \bigg), 
    \end{align*}
    Plugging into \eqref{induct0}, we arrive at 
    \begin{equation} \label{inductstep}
         \begin{aligned} 
        & \quad \E [\exp (S_t) | \mathcal{F}_{\tau_t} ] \\&  \leq \exp\bigg( (w_{t+1} - w_{T+1}) {\DP}^{\mu^*}_{\mu^t} + C \sum_{i=t+1}^T w_{i+1} \eta^2 \bigg)  \\& \quad \E [\exp(w_{t+1}(\sum_{\tau \in \mathcal{D}_t} \xi_{\tau } + \frac{2 \eta^2 d }{ \sigma_{\Psi}}  \| z^t_m\|^2 ))| \mathcal{F}_{\tau_t}] .
    \end{aligned}
    \end{equation}
    The rest of the business is to take care of the conditional expectation term.
    
    Since $z^{\tau}$ is an i.i.d. process, we can proceed with the fact that $\E [  \langle z^{\tau}, \mu^* - \mu^{\tau_t}\rangle | \mathcal{F}_{\tau_t} ] = 0 $ for any $\tau \in \mathcal{D}_t$. 
    Through Taylor's expansion,
    
    \begin{align*}
     & \quad  \E [ \exp(w_{t+1}  (\sum_{\tau \in \mathcal{D}_t } \eta \langle z^{\tau}, \mu^* - \mu^{\tau} \rangle +  \frac{2\eta^2d}{\sigma_{\Psi}} \|z^t_m\|^2 ))  | \mathcal{F}_{\tau_t}] 
    \\ & =  \E [\exp(w_{t+1} \underbrace{\sum_{\tau \in \mathcal{D}_t} \eta \langle z^{\tau}, \mu^* - \mu^{\tau_t}  \rangle }_{V_{1,t}}
    \\ & \quad  + w_{t+1} (\underbrace{\sum_{\tau \in \mathcal{D}_t } \eta \langle z^{\tau}, \mu^{\tau_t} - \mu^{\tau}  \rangle
     + \frac{2\eta^2 d}{\sigma_{\Psi}} \|z^t_m\|^2) )}_{V_{2,t} }| \mathcal{F}_{\tau_t} ]
    \\ & =    \E [ 1 +  w_{t+1}V_{2,t} + \sum_{i =2}^{\infty }\frac{1}{i!} (w_{t+1}( V_{1,t} + V_{2,t} ) )^i | \mathcal{F}_{\tau_t}] ,
    \end{align*}
    where we leveraged that $\E [w_{t+1}V_{1,t} | \mathcal{F}_{\tau_t}] = 0$.
    \begin{align*}
        & \sum_{\tau \in \mathcal{D}_t } \eta \langle z^{\tau}, \mu^{\tau_t} - \mu^{\tau} \rangle  \leq  \eta \| z^t_m\| \sum_{\tau \in \mathcal{D}_t } \|\mu^{\tau_t} - \mu^{\tau}\|  
        \\ & \leq  \frac{\eta^2}{\sigma_{\Psi}} \|z^t_m \| \sum_{\tau \in \mathcal{D}_t} \sum_{ r \in \cup_{s = \tau_t}^{\tau}\mathcal{D}_s }\| \ell^r \| \leq \frac{2d^2\eta^2}{\sigma_{\Psi}} \|z^t_m \|  (L + \| z^t_m\|) 
        \\ & = \frac{ 2d^2 \eta^2}{ \sigma_{\Psi}}\|z^t_m\|^2 + \frac{2 d^2\eta^2 L}{ \sigma_{\Psi}}  \|z^t_m\|.
        \end{align*}
Hence, $
 V_{2, t} \leq \frac{  (2+2d)d \eta^2}{ \sigma_{\Psi}}\|z^t_m\|^2 + \frac{2 d^2\eta^2 L}{ \sigma_{\Psi}}  \|z^t_m\|  \leq  \frac{5d^2 \eta^2}{\sigma_{\Psi}} \|z^t_m\|^2 + \frac{d^2 \eta^2}{ \sigma_{\Psi}} \kappa^2 \sigma^2.$
By Lemma \ref{chainsum}, we have,
    \begin{align*}
        & \quad   V_{1,t} + V_{2, t} 
       \\ &  \leq  \|z^t_m\| \sum_{\tau \in \mathcal{D}_t}\eta \| \mu^*- \mu^{\tau}\|    + \frac{5\eta^2 d^2}{\sigma_{\Psi}} \|z^t_m\|^2 +\frac{d^2\eta^2}{\sigma_{\Psi}} \kappa^2 \sigma^2
         \\ & \leq \sqrt{\frac{2d^2\eta^2 }{  \sigma_{\Psi}} {\DP}^{\mu^*}_{\mu^t}} \|z^t_m\| +  \frac{  2d^2 \eta^2 }{\sigma_{\Psi}} \kappa^2 \sigma^2  + \frac{8 \eta^2 d^2}{\sigma_{\Psi}} \| z^t_m\|^2. 
    \end{align*}
Thus,
    \begin{align*}
     & \quad  \E [ 1 +  w_{t+1}V_{2,t} + \sum_{i =2}^{\infty }\frac{1}{i!} (w_{t+1}( V_{1,t} + V_{2,t} ) )^i | \mathcal{F}_{\tau_t}]
    \\& \leq  \E\bigg[ 1 + \underbrace{w_{t+1}\frac{5\eta^2 d^2}{\sigma_{\Psi}}}_{ \leq b^2}\|z^t_m\|^2  + \underbrace{w_{t+1}\frac{d^2 \eta^2 }{\sigma_{\Psi}}\kappa^2 }_{ \leq c^2}\sigma^2 \\ &  + \sum_{i =2}^{\infty } \frac{1}{i!} (\underbrace{w_{t+1} \sqrt{\frac{2d^2\eta^2 }{  \sigma_{\Psi}} {\DP}^{\mu^*}_{\mu^t}} }_{a} \| z^t_m\| \\ & + \underbrace{w_{t+1}\frac{2d^2 \eta^2 }{\sigma_{\Psi}}\kappa^2 }_{c^2}\sigma^2 + \underbrace{w_{t+1}\frac{8\eta^2 d^2}{\sigma_{\Psi}}}_{b^2}\|z^t_m\|^2)^i | \mathcal{F}_{\tau_t}\bigg].
    \end{align*}
Applying Lemma \ref{subgauss} we obtain
\begin{align*}
 & \text{R.H.S} \leq  \exp\bigg( 324d \big( w_{t+1}^2\frac{2d^2\eta^2}{\sigma_{\Psi}}{\DP}^{\mu^*}_{\mu^t}  \\ 
 & \quad  + w_{t+1}\frac{\eta^2 d^2}{\sigma_{\Psi}}(  8 + \kappa^2) \big) \sigma^2 \bigg), 
\end{align*} 
under the conditions that $w_{t+1} \frac{8\eta^2 d^3}{\sigma_{\Psi}} \leq \frac{1}{432 \sigma^2}$.
Incorporating it into \eqref{inductstep}, we have,
    \begin{align*}
         & \quad \E [\exp(S_t) | \mathcal{F}_{\tau_t} ]
        \\ &  \leq \exp( (w_{t+1} + \frac{648 w^2_{t+1} d^3}{\sigma_{\Psi} \eta^{-2}\sigma^{-2}}  - w_{T+1}) {\DP}^{\mu^*}_{\mu^t}  + C \sum_{i=t+1}^{T+1} w_{i} \eta^2 )) \\ 
        & \leq \exp\bigg( (w_t - w_{T+1}) {\DP}^{\mu^*}_{\mu^t} + C \sum_{i=t}^T w_{i+1} \eta^2 \bigg), 
    \end{align*}
with probability $1$. Thus the statement follows.

\end{proof}

\begin{lemma} \label{maximumiid}
  For all $t = 1, \ldots, T$, $ \|z^t_m\|$ satisfies \eqref{maxiiid} conditioned on $\mathcal{F}_{\tau_t}$, for all $\lambda \in \R$ such that $|\lambda| \leq \frac{1}{\sqrt{108d}\sigma}$, we have
  \begin{equation}
      \label{maxiiid}
       \E \left[ \exp (  \lambda^2 \| z^t_m \|^2  | \mathcal{F}_{\tau_t} \right] \leq  \exp(216 d\lambda^2 \sigma^2).
  \end{equation}
\end{lemma}


 \begin{proof}
 Let $Z = \|z^t_m\|$ and $X = \|z\|$, by the union bound,
    \begin{align*}
   & \mathbb{P} \{\|z\| \geq t\} \leq \E [ \exp\frac{\|z\|^2}{\sigma^2}] \exp( - \frac{t^2}{\sigma^2}) \leq  e \exp(-\frac{t^2}{\sigma^2})
    \\ & \mathbb{P} \{ Z \geq t\} \leq 2de \exp(-\frac{t^2}{\sigma^2})
.   \end{align*}
Using the fact that $\Gamma(x) \leq 3 x^{x} \forall x \geq 1 $ for $p \geq 1$:
\begin{align*} 
\mathbb{E}|Z|^{p} & =\int_{0}^{\infty} \mathbb{P}\left\{|Z|^{p} \geq u\right\} d u 
 =\int_{0}^{\infty} \mathbb{P}\{|Z| \geq t\} p t^{p-1} d t 
\\ & \leq \int_{0}^{\infty} 2de e^{-\frac{t^2}{\sigma^2}} p t^{p-1} d t 
= de \sigma^p p \Gamma(p / 2) \quad \text { set } \frac{t^2}{\sigma^2}=s
\\ & \leq 3 de \sigma^p p(p / 2)^{p / 2} .
\end{align*}
By Taylor expansion and Stirling formula, $p! \geq (p/e)^{p}$,
\begin{align*}
    \E [\exp(\lambda^2 Z^2) ] & = \mathbb{E}\left[1+\sum_{p=1}^{\infty} \frac{\left(\lambda^{2} Z^{2}\right)^{p}}{p !}\right]=1+\sum_{p=1}^{\infty} \frac{\lambda^{2 p} \mathbb{E}\left[Z^{2 p}\right]}{p !} 
    \\  \leq 1 + \sum_{p=1}^{\infty} &\cfrac{ 6dep (\lambda^2 \sigma^2 p)^p}{ (p / e)^p}
    \\  \leq 1 + \sum_{p=1}^{\infty} &\cfrac{ 6d(\lambda^2 \sigma^2 \cdot 3 e \cdot 2p )^p}{p^p}  \leq 1 + \sum_{p=1}^{\infty} ( 108 d \lambda^2 \sigma^2 )^p
   \\ & \leq \frac{1}{1 - 108d \lambda^2 \sigma^2} 
     \leq \exp( 216 d \lambda^2 \sigma^2) ,
\end{align*}
when $\lambda^2 \leq \frac{1}{108 d \sigma^2 }$, where we use the fact that the convergent geometric series $\frac{1}{1 - x} \leq \exp (2x)$.

 \end{proof}
 

\begin{lemma} \label{chainsum}
We have, for all $t = 1, \ldots, T$, conditioned on $\mathcal{F}_{\tau_t}$, 
 \begin{equation}
  \sum_{\tau \in \mathcal{D}_t } \eta\| \mu - \mu^{\tau } \| 
  \leq \sqrt{\frac{2d^2\eta^2 }{  \sigma_{\Psi}} {\DP}^{\mu^*}_{\mu^t}} +  \frac{  2d^2 \eta^2 }{\sigma_{\Psi}} (L + \| z^t_m\|)
 \end{equation}    
\end{lemma}

\begin{proof}
By applying Cauchy-Schwarz inequality, $\sigma_{\Psi}$-strong convexity of Bregman divergence, and  triangular inequality for a certain number of times to the term $\sum_{\tau \in \mathcal{D}_t } \eta \| \mu^* - \mu^{\tau } \| \leq  \sum_{\tau \in \mathcal{D}_t} \eta (\|\mu^*- \mu^t\| + \sum_{s \in \mathcal{D}_t} \|\mu^{s+1} - \mu^{s}\| )$, bounded by
 \begin{align*}
 &   \sum_{\tau \in \mathcal{D}_t} \sqrt{\frac{2 \eta^2 }{  \sigma_{\Psi}}  {\DP}^{\mu^*}_{\mu^t}} +  \sum_{ \tau \in \mathcal{D}_t }\sum_{s \in \mathcal{D}_t} \frac{\eta}{\sigma_{\Psi}} \| \sum_{r \in \mathcal{D}_s} \eta \ell^r\|
\\ & \leq  \sqrt{\frac{2d^2\eta^2 }{  \sigma_{\Psi}} {\DP}^{\mu^*}_{\mu^t}} +  \frac{  2d^2 \eta^2 }{\sigma_{\Psi}} (L + \| z^t_m\|).
 \end{align*}
 with probability $1$ since it holds for any $z^{\tau}$, $\tau$ generated after $ \mathcal{D}_{\tau_t}$, conditioned on $\mathcal{F}_{\tau_t}$.

\end{proof}

\begin{lemma}\label{subgauss}[technical, adapted from \cite{ene2022high}]
For $a \geq 0 $, $ b^2 \leq \frac{1}{432d\sigma^2}$, $c \in \R$, and a non-negative random variable $Z$ satisfying inequality \eqref{maxiiid},
\begin{equation} \label{technical}
    \begin{aligned}
         & \quad \mathbb{E}\left[1+b^{2} Z^{2} + c^2\sigma^2 +\sum_{i=2}^{\infty} \frac{1}{i !}\left(a Z+b^{2} Z^{2} + c^2 \sigma^2 \right)^{i}\right] 
    \\ & \leq \exp \left(324d \left(a^{2}+b^{2} + c^2  \right) \sigma^{2}\right). 
    \end{aligned}
\end{equation}
where the expectation is taken conditioned on $\mathcal{F}_{\tau_t}$.
\end{lemma}

 \begin{proof}
  For $a \geq 0 $, $0 \leq b \leq \frac{1}{\sqrt{432d\sigma^2}}$, we consider two cases $a \geq \frac{1}{2\sqrt{108d\sigma^2}}$ and $a \leq \frac{1}{2\sqrt{108d\sigma^2}}$, and use $xy \leq \frac{x^2}{432d\sigma^2} + (108d\sigma^2)y^2$, separating $c^2 \sigma^2$ out we have, when $a \geq \frac{1}{2\sqrt{108d\sigma^2}}$,
 \begin{align*}
     & \mathbb{E}\left[1+b^2 Z^2 + c^2 \sigma^2+\sum_{i=2}^{\infty} \frac{1}{i !}\left(a Z+b^2 Z^2 + c^2 \sigma^2 \right)^i\right] \leq 1 + 
     \\ & \mathbb{E}[b^2 Z^2 + c^2 \sigma^2+\sum_{i=2}^{\infty} \frac{((\frac{1}{432d \sigma^2} + b^2) Z^2 + (a^2+c^2) \sigma^2)^i}{i !}]
     \\ & \leq \E \left[ \exp((\frac{1}{432d \sigma^2} + b^2) Z^2 + (108 d a^2+c^2)\sigma^2) \right]
     \\ & \leq \exp((\frac{1}{432d \sigma^2} + b^2) {\sigma}^2 + (108 d a^2 + c^2)\sigma^2) .
     \\ & \leq \exp(216d(a^2 + b^2) {\sigma}^2 + (108 d a^2 + c^2 )\sigma^2)
 \end{align*} 
 Now let $a \leq \frac{1}{2\sqrt{108d}\sigma}$ and $m= \max\{a, |b|\}$. 
 We have
 \begin{align*}
      & \quad\mathbb{E}\left[1+b^2 Z^2 + c^2 \sigma^2 +\sum_{i=2}^{\infty} \frac{1}{i !}\left(a Z+b^2 Z^2 + c^2 \sigma^2 \right)^i\right] 
      \\ & = \E \left[\exp (aZ + b^2 Z^2 + c^2 \sigma^2 ) - aZ \right] 
      \\ & \leq \E [(aZ + \exp (a^2Z^2)) \exp (b^2Z^2 + c^2 \sigma^2 ) - aZ]
      \\ & \leq  \E [\exp((a^2 + b^2)Z^2 + c^2\sigma^2) 
      \\ & \quad + \sqrt{m^2Z^2 + c^2\sigma^2} (\exp( m^2Z^2 + c^2\sigma^2) - 1) ] \leq 
      \\ &   \E [\exp((a^2 + b^2)Z^2 + c^2 \sigma^2) + \exp( 2m^2Z^2 + 2c^2\sigma^2 ) - 1]
      \\ & \leq \exp(216d (a^2 +b^2 +2m^2)\sigma^2  +2c^2\sigma^2)] 
      \\ & \leq \exp(216d(a^2 +b^2 +2m^2 + 2c^2 )\sigma^2),
 \end{align*}
 where we have leveraged two facts, $e^x \leq x + e^{x^2}$, and $x(e^{x^2} - 1) \leq e^{2x^2} - 1 \ \forall x$.
 Sum the two upper estimates together and we get \eqref{technical}.
 
 \end{proof}

\end{document}